%% file: main.tex
\documentclass[11pt, letterpaper]{article}

\usepackage[
letterpaper,
top=1.2in,
bottom=1.2in,
left=1in,
right=1in]{geometry}

\usepackage[dvipsnames]{xcolor}

\usepackage{amsthm, amsmath, amssymb}
\usepackage{mathtools}
\usepackage{tikz}
\usepackage{todonotes}
\usepackage{pgfplots}
\pgfplotsset{compat=1.18}

\usepackage{lipsum}
\usepackage{nicefrac}       
\usepackage{microtype}      
\usepackage{tabularx}
\usepackage[T1]{fontenc}
\usepackage[full]{textcomp}
\usepackage[american]{babel}
\usepackage{booktabs}
\usepackage{xparse}

\usepackage{paralist}
\usepackage{mdframed}
\usepackage{tikz}
\usepackage{caption}
\usepackage{array}

\usepackage{newpxtext}
\usepackage{textcomp} 
\usepackage[varg,bigdelims]{newpxmath}

\usepackage[colorlinks=true,
linkcolor=blue,
urlcolor=blue,
citecolor=purple,
linktocpage=true]{hyperref} 

\usepackage[capitalise,nameinlink]{cleveref}
\crefname{lemma}{Lemma}{Lemmas}
\crefname{fact}{Fact}{Facts}

\newtheorem{fact}{Fact}
\newtheorem{theorem}{Theorem}
\newtheorem{lemma}{Lemma}
\newtheorem{definition}{Definition}
\newtheorem{proposition}{Proposition}
\newtheorem{example}{Example}

\newcommand{\M}{\mathcal{M}}
\newcommand{\U}{\mathbf{U}}
\newcommand{\D}{\mathbf{D}}
\newcommand{\Lin}{\mathcal{L}}
\newcommand{\ptime}{\mathrm{P}}
\newcommand{\np}{\mathrm{NP}}
\newcommand{\fptime}{\mathrm{FP}}
\newcommand{\shp}{\mathrm{\#P}}
\newcommand{\Knapsack}{\textsc{Knapsack}}
\newcommand{\opt}{\operatorname{\textnormal{\textsc{Min}}}}

\DeclareMathOperator{\comp}{Comp}

\usepackage{algorithm}
\usepackage{algorithmic}

\usepackage{natbib}

\usepackage{mathabx}

\input{math_commands.tex}

\usepackage{bm}
\usepackage{microtype}

\newcommand{\FormatAuthor}[3]{%
    \centering
    \begin{tabular}{c}
    #1 \\
    {\small\texttt{#2}} \\
    {\small\shortstack[c]{#3}}%
    \end{tabular}%
}

\title{Probabilistic Explanations for Linear Models\thanks{
    The authors
opted for randomized author ordering, denoted by \textcircled{r},
instead of alphabetical. A verifiable record of the randomization
is available at \protect\url{https://www.aeaweb.org/journals/policies/random-author-order/search}}}
\author{\centering
    \begin{tabular}[h!]{p{5cm} p{5cm} p{5cm}}
        \FormatAuthor{Bernardo Subercaseaux \textcircled{r}}{\url{bersub@cmu.edu}}{Carnegie Mellon University} &
        \FormatAuthor{Marcelo Arenas \textcircled{r}}{\url{marenas@ing.puc.cl}}{PUC Chile\\ IMFD Chile\\ RelationalAI} &
        \FormatAuthor{Kuldeep S. Meel \textcircled{r}}{\url{meel@cs.toronto.edu}}{Georgia Institute of Technology\\ University of Toronto}
  \end{tabular}
}

\begin{document}
\maketitle

\begin{abstract}
    Formal XAI is an emerging field that focuses on providing explanations with mathematical guarantees for the decisions made by machine learning models. A significant amount of work in this area is centered on the computation of ``sufficient reasons''. 
    Given a model $\M$ and an input instance $\vx$, a sufficient reason for the decision $\M(\vx)$ is a subset $S$ of the features of $\vx$ such that for any instance $\vz$ that has the same values as $\vx$ for every feature in $S$, it holds that $\M(\vx) = \M(\vz)$. 
    Intuitively, this means that the features in $S$ are sufficient to fully justify the classification of $\vx$ by $\M$.
    For sufficient reasons to be useful in practice, they should be as small as possible, and a natural way to reduce the size of sufficient reasons is to consider a probabilistic relaxation; the probability of $\M(\vx) = \M(\vz)$ must
    be at least some value $\delta \in (0,1]$, for a random instance $\vz$ that coincides with $\vx$ on the features in $S$.  Computing small $\delta$-sufficient reasons ($\delta$-SRs) is known to be a theoretically hard problem; even over decision trees — traditionally deemed simple and interpretable models — strong inapproximability results make the efficient computation of small $\delta$-SRs unlikely.
    We propose the notion of $(\delta, \epsilon)$-SR, a simple relaxation of $\delta$-SRs, and show that this kind of explanations can be computed efficiently over linear models. 
\end{abstract}

\section{Introduction}
\label{sec-into}
\input{sections/intro}

\section{Probabilistic Sufficient Reasons}
\label{sec-prop-SR}
\input{sections/prob-sufficient-reasons}

\section{Approximating \texorpdfstring{$\delta$}{delta}-Sufficient Reasons}
\label{sec-comp-problem}

\input{sections/new-computational-problems}

\section{Locally Minimal Probabilistic Explanations}
\label{sec-loc-min}
\input{sections/locally-minimal}

\section{Approximations on the Size of Explanations}
\label{sec-size-approx}
\input{sections/size-approximations}

\section{Conclusions and Future Work}
\label{sec-conclusions}
\input{sections/conclusions}
\bibliographystyle{apalike}
\bibliography{references}

\newpage
\onecolumn
\appendix
\section*{Appendix}
\input{sections/appendix}

\end{document}

%% file: math_commands.tex
\usepackage{amsmath,amsfonts,bm}

\def\eqref#1{equation~\ref{#1}}

\def\1{\bm{1}}

\def\vw{{\bm{w}}}
\def\vx{{\bm{x}}}
\def\vy{{\bm{y}}}
\def\vz{{\bm{z}}}

\DeclareMathAlphabet{\mathsfit}{\encodingdefault}{\sfdefault}{m}{sl}
\SetMathAlphabet{\mathsfit}{bold}{\encodingdefault}{\sfdefault}{bx}{n}

\newcommand{\E}{\mathbb{E}}



%% file: sections/intro.tex
Explaining the decisions of Machine Learning classifiers is a fundamental problem in XAI (Explainable AI), and doing so with formal mathematical guarantees on the quality, size, and semantics of the explanations is in turn the core of \emph{Formal XAI}~\citep{formal-xai}. 
Within formal XAI, one of the most studied kinds of explanations is that of \emph{sufficient reasons}~\citep{Darwiche_Hirth_2020}, which aim to explain a decision $\M(\vx) = 1$ by presenting a subset $S$ of the features of the input $\vx$ that implies $\M(\vz) = 1$  for any $\vz$ that agrees with $\vx$ on $S$. 
In the language of theoretical computer science, these correspond to \emph{certificates} for $\M(\vx)$.

\begin{example}
Consider a binary classifier~$\M$ defined as 
	\[
	\M(\vx) = \left(x_1 \lor \overline{x_3}\right) \land   \left(x_2 \lor \overline{x_1}\right) \land \left(x_4 \lor x_3\right),
	\]
	and the input instance $\vx = \left( 1, \,  1, \, 0, \, 1 \right)$. We can say that $\M(\vx)$ ``because'' $x_1 = 1, x_2 = 1$, and $x_4 = 1$, as they are sufficient to determine the value of $\M(\vx)$ regardless of $x_3$.
	\label{ex:sufficient-reason}
\end{example}

Let us start formalizing the framework for our work.  First, we consider binary boolean models $\M\colon \{0, 1\}^d \to \{0, 1\}$. Despite our domain being binary, we will need a third value, $\bot$, to denote \emph{``unknown''} values.  For example, we may represent a person who \emph{does} have a car, \emph{does not} have a house, and for whom we do not know if they have a pet or not, as $\left(1, \, 0, \, \bot\right)$. 
We say elements of $\{0, 1, \bot\}^d$ are \emph{partial instances}, while elements of $\{0, 1\}^d$ are simply \emph{instances}. To illustrate, in~\Cref{ex:sufficient-reason} we used the partial instance $\vy = \left(1, \, 1, \, \bot, \, 1 \right)$ to explain $\M(\vx) = 1$.
We use the notation $\vy \subseteq \vx$ to denote that the (partial) instance $\vx$ \emph{``fills in''} values of the partial instance $\vy$; more formally, we use $\vy \subseteq \vx$ to mean that $y_i = \bot \lor y_i = x_i$ for every $i \in [d]$. Finally, for any partial instance $\vy$ we denote by $\comp(\vy)$ the set of instances $\vx$ such that $\vy \subseteq \vx$, thinking of $\comp(\vy)$ as the set of \emph{completions} of $\vy$. One can define sufficient reasons as follows with this notation.

\begin{definition}[Sufficient Reason~\citep{Darwiche_Hirth_2020}]
	We say $\vy$ is a \emph{sufficient reason} for $\vx$ if for any completion $\vz \in \comp(\vy)$ it holds that $\M(\vx) = \M(\vz)$.
	\label{def:sufficient-reason}
\end{definition}
A crucial factor for the helpfulness of sufficient reasons as explanations is their size; even though $\vx$ is always a sufficient reason for its own classification, we long for explanations that are much smaller than $\vx$ itself. \citet{millerMagicalNumberSeven1956}, for instance, goes on to say that explanations consisting of more than $9$ features are probably too large for human stakeholders. In general, empirical research suggests that explanations ought to be small~\citep{Narayanan_Chen_He_Kim_Gershman_Doshi-Velez_2018, Lage_Chen_He_Narayanan_Kim_Gershman_Doshi-Velez_2019}.
There are several ways of formalizing the succinctness we desire for sufficient reasons:

\begin{itemize}
    \item \textbf{(Minimum Size)} For a sufficient reason $\vy$, we define its \emph{explanation size} $|\vy|_e$ as the number of defined features in $\vy$, or equivalently, $|\vy|_e := d - |\vy|_\bot$, where $|\vy|_\bot$ is the number of features of $\vy$ taking $\bot$. See e.g.,~\cite{NEURIPS2020_b1adda14}.\footnote{When talking about a partial instance $\vy$, we will use the ``size'' of $\vy$ to mean $|\vy|_e$.}
    \item \textbf{(Subset minimality)} We say a sufficient reason $\vy$ for a pair $(\M, \vx)$ is \emph{minimal} if there is no other sufficient reason $\vy'$ for $(\M, \vx)$ such that $\vy' \subsetneq \vy$. In fact, the original definition of sufficient reasons of~\citet{Darwiche_Hirth_2020} includes minimality as a requirement, and so is the case under the \emph{``abductive explanation''} naming~\citep{Ignatiev_Narodytska_Asher_Marques-Silva_2021}.
    \item \textbf{(Relative to average explanation)} \citet{blanc2021provably} compute explanations that are small relative to the \emph{``certificate complexity''} of the classifier $\M$, meaning the average size of the minimum sufficient reason where the average is taken over all possible instances $\vx$.
\end{itemize}

Nevertheless, there is a path toward even smaller explanations: \emph{probabilistic} sufficient reasons~\citep{Waldchen_MacDonald_Hauch_Kutyniok_2021, Izza_Huang_Ignatiev_Narodytska_Cooper_Marques-Silva_2023}. 
As will be shown in Example 2., and is noted as a remark by~\citet{blanc2021provably}, these can be arbitrarily smaller than minimum size sufficient reasons.

%% file: sections/prob-sufficient-reasons.tex
The main idea of probabilistic sufficient reasons is to relax the condition \emph{``all completions of the explanation $\vy$ have the same class as $\vx$''} to \emph{``a random completion of $\vy$ has the same class as $\vx$ with high probability''}.

 Let us use notation $\vz \sim \U(\vy)$ to denote that $\vz$ is a completion of $\vy$ drawn uniformly at random.
With this notation we can define $\delta$-sufficient reasons:\footnote{Also known as $\delta$-relevant sets~\cite{Izza2021EfficientEW,Waldchen_MacDonald_Hauch_Kutyniok_2021}.}
\begin{definition}[\cite{Waldchen_MacDonald_Hauch_Kutyniok_2021}]
	For any $\delta \in [0, 1]$, a $\delta$-sufficient reason ($\delta$-SR) for an instance $\vx$, is a partial instance $\vy \subseteq \vx$ such that
	\[
		\Pr_{\vz \sim  \U(\vy)}\big[\M(\vz) = \M(\vx)\big] \geq \delta.
	\]
	\label{def:delta-SR}
\end{definition}
Naturally, a minimum $\delta$-SR is a $\delta$-SR of minimum size. Note immediately that~\Cref{def:delta-SR} and~\Cref{def:sufficient-reason} coincide when $\delta = 1$.
\subsection{\texorpdfstring{The size of $\delta$-SRs}{The size of delta-SRs}}

Interestingly, even a $0.999999$-SR can be arbitrarily smaller, in terms of defined features, than the smallest sufficient reason (i.e., $1$-SR) for a pair $(\M, \vx)$, even when $\M$ is a linear model, as we will illustrate in~\Cref{ex:delta-sr-size}. Before providing the example, let us define linear models.

\begin{definition}
	A (binary) linear model $\Lin$ of dimension $d$ is a pair $(\vw, t)$, where $\vw \in \mathbb{Q}^d$ and $t \in \mathbb{Q}$. Its classification over an instance $\vx$ is defined simply as 
	\[
		\Lin(\vx) = \begin{cases}
			1 & \text{if } \vx \cdot \vw \geq t\\
			0 & \text{otherwise}.
		\end{cases}
	\]
	\label{def:linear-models}
\end{definition}



\begin{example}
Consider a linear model $\Lin$ of dimension $d= 1000$ with parameters $t = 1250$ and
$$
	\vw = (
		1000, \, 1, \, 1, \,  1, \,  1, \,  \ldots, \, 1
	).
$$
Let the instance $\vx$ be 
	$(
		1, \, 1, \, 1, \,  1, \,  1, \,  \ldots, \, 1
	),$ so that clearly $\Lin(\vx) = 1$.
One can easily see that any $1$-SR for $\vx$ under $\Lin$ has size $251$, as it must include the first feature and any $250$ other features.
However, if we consider $\vy = (
		1, \, \bot, \, \bot, \,  \bot,  \ldots, \, \bot
	)$, then a simple application of the Chernoff-Hoeffding concentration bound (in the appendix for the completeness) gives that
\[
	\Pr_{\vz \sim \U(\vy)}\Big[ \Lin(\vz) = 1\Big] \geq  0.999999.\]
	This suggests that we might say $\Lin(\vx) = 1$ ``because'' $x_1 = 1$; formally, $\vy$ is a $0.999999$-SR, and $251$ times smaller than any $1$-SR for $\Lin(\vx)$.
\label{ex:delta-sr-size}
\end{example}


%% file: sections/new-computational-problems.tex
Unfortunately, computing small $\delta$-SRs is computationally challenging, even when attempting to find approximate solutions. Let us contextualize our main result by summarizing first what is known about the complexity of computing $\delta$-SRs and their deterministic predecessors, $1$-SRs.  

\citet{NEURIPS2020_b1adda14} showed that computing a minimum $1$-SR is $\Sigma_2^p$-hard for neural networks, $\np$-hard for decision trees, and polynomial-time solvable for linear models. Then, \citet[Theorem 2.4]{Waldchen_MacDonald_Hauch_Kutyniok_2021} showed that computing minimum $\delta$-SRs for neural networks is hard for $\np^{\mathrm{PP}}$, and \citet{NEURIPS2022_b8963f6a} proved that even for the restricted class of decision trees, which are usually considered interpretable, minimum $\delta$-SRs cannot be computed in polynomial time unless $\ptime = \np$ (and neither can subset-minimal $\delta$-SRs for $\delta < 1$, in contrast to the $\delta = 1$ setting which is in $\ptime$~\citep{izzaExplainingDecisionTrees2020,roaModelInterpretabilityLens2020}). 
For linear models, even computing the value
\[
    \Pr_{\vz \sim \U(\vy)}\big[\Lin(\vz) = \Lin(\vx)\big]
\]
exactly is $\shp$-hard~\citep{NEURIPS2020_b1adda14}, from where the following is easy to show.\footnote{\citet{izzaComputingProbabilisticAbductive2023} already made a more general observation of this form, but did not provide a hardness result for linear models.}\footnote{All proofs that are not presented in the main text can be found in the appendix.}
 
\begin{proposition}\label{prop:hardness}
    Given a linear model $\Lin$, an instance
    $\vx$, and $\delta \in [0,1]$, the size of the smallest $\delta$-SR for $(\Lin, \vx)$
    cannot be computed in polynomial time unless $\fptime = \shp$.
\end{proposition}

Furthermore, the situation does not improve if we aim to efficiently approximate the value $\opt(\M, \vx, \delta)$. \citet[Theorem 2.5]{Waldchen_MacDonald_Hauch_Kutyniok_2021} studied general classifiers (e.g., neural networks) and showed that no algorithm can achieve an approximation factor of $d^{1-\alpha}$ for this problem, where $d$ is the dimension of the classifier and $\alpha > 0$, unless $\ptime = \np$. \citet{Kozachinskiy_2023} proved that this approximation task is also hard for decision trees.

However, these hardness results do not preclude the existence of
efficient algorithms for computing or approximating $\delta$-SR for
linear models. Hence, the goal of this section is to explore these
questions for such models, given their practical importance.

\subsection{\texorpdfstring{A Simple Relaxation: $(\delta, \varepsilon)$-min-SR}{A Simple Relaxation: (delta, epsilon)-min-SR}}

In light of the hardness results for $\delta$-SRs, it is natural to consider a further relaxation that would allow for tractability. Consider for instance a customer of a bank who wants a $0.95$-SR for why their application for a loan was rejected. Such an explanation would consist of a small number of features of their application profile that are relevant to the decision since $95\%$ of applicants with such a profile would also get rejected. We expect that, in such a scenario, the user would not particularly care if the explanation she obtains holds 
for $95\%$ of potential applicants or for $94.9997\%$ of them. In other words, the value of $\delta$ is chosen in a trade-off between the size of the explanation and the desired level of confidence or ``explanation power''. We posit that in such a trade-off, the user is more sensitive to increases in the explanation size than they are to a minor perturbation in $\delta$, the probability guarantee.  This motivates the following definition:


\begin{definition}[$(\delta, \varepsilon)$-min-SR]
    Given a model $\M$, an instance $\vx$, and values $\delta, \varepsilon \in (0, 1)$, we say a partial instance $\vy$ of size is a $(\delta, \varepsilon)$-min-SR if there exists a value $\delta^\star \in [\delta - \varepsilon, \delta + \varepsilon]$ such that
    $\vy$ is a minimum $\delta^\star$-SR for $\vx$ under $\M$.
\end{definition}

Note that, even though the guarantee of a $(\delta, \varepsilon)$-min-SR is symmetric around $\delta$, our definition is such that the ability of efficiently computing $(\delta, \varepsilon)$-min-SRs is enough for the following two tasks:
\begin{enumerate}
    \item A user wants an explanation as small as possible and of probability ``close'' to $\delta$. Then, by computing a~$(\delta-\varepsilon/2, \varepsilon/2)$-min-SR, they obtain an explanation whose probability guarantee is at most $\varepsilon$ away from $\delta$, and is no larger in size than the minimum $\delta$-SR.
    \item The owner of the model wants to offer a $\delta$-SR that is as small as possible to a customer, and they want to be strict on the $\delta$ part, since offering a $(\delta - \varepsilon)$-SR would be misleading and could lead to legal issues. Then, by computing a $(\delta+\varepsilon/2, \varepsilon/2)$-min-SR, they can guarantee that the explanation is at least $\delta$-SR, while still being likely much smaller than a minimum $1$-SR. 
\end{enumerate}

The inapproximability result of~\citet{Kozachinskiy_2023} can be translated to the $(\delta, \varepsilon)$-min-SR problem as follows:
\begin{theorem}[\citet{Kozachinskiy_2023}, Theorem 1]
    Unless $\mathrm{SAT}$ can be solved in quasi-polynomial time, one cannot compute a $(\delta, \varepsilon)$-min-SR for decision trees in polynomial time, and furthermore, any polynomial-time algorithm that guarantees to provide a $\delta'$-SR for some $\delta' \in [\delta-\varepsilon, \delta+\varepsilon]$ will produce explanations that are up to $\Omega(d^{1-\alpha})$ times larger than any $(\delta, \varepsilon)$-min-SR, for any $\alpha > 0$.
\end{theorem}
Note that this hardness result for decision trees implies in turn hardness for neural networks by using standard compilation techniques~\citep{NEURIPS2020_b1adda14}.
Our main result is that, for linear models, we can efficiently compute $(\delta, \varepsilon)$-min-SRs, making them the first class of models for which we have such a positive result. To state our runtime more cleanly, we use the standard notation $\widetilde{O}(f)$ to mean $O(f \cdot \log(f)^c)$ for some constant $c \in \mathbb{R}$. 
\begin{theorem}
    \label{prop:smoothed-explanation}
    Given a linear model $\Lin$ and an input $\vx$, we can compute a $(\delta, \varepsilon)$-min-SR successfully with probability $1 - \gamma$ in time $\widetilde{O}\left( \frac{d}{\varepsilon^2\gamma^2}\right)$; that is, polynomial in $d$, $1/\varepsilon$, and $1/\gamma$.
\end{theorem}

We remark that previous approaches for computing approximate probabilistic explanations lacked theoretical guarantees on the size of the explanations produced~\citep{izzaComputingProbabilisticAbductive2023,Izza2021EfficientEW,izza2024locallyminimalprobabilisticexplanations}.





In order to prove~\Cref{prop:smoothed-explanation} we will need two main ideas: (i) the fact that we can estimate the probabilities of models accepting a partial instance through Monte Carlo sampling\footnote{This idea is already used in the work of~\citet{izza2024locallyminimalprobabilisticexplanations}}, and (ii) that under the uniform distribution it is easy to decide which features ought to be part of small explanations.

\subsection{Estimating the Probability of Acceptance}
 The  hardness of computing
$\Pr_{\vz  \sim \U(\vy)}[\M(\vz) = 1]$ is about computing it to arbitrarily high precision, i.e., with an additive error within $O(2^{-d})$. However, computing a less precise estimation of $\Pr_{\vz \sim \U(\vy)}[\M(\vz) = 1]$ is simple, as the next fact (which is a direct consequence of~\citet{hoeffdingProbabilityInequalitiesSums1963a}'s inequality)  states.

\begin{fact}\label{fact:hoeffding}
    Let $f$ be an arbitrary boolean function on $n$ variables. Let $M$ be any positive integer,
    and let $\vx_1, \ldots, \vx_M$ be $M$ uniformly random samples from $\{0, 1\}^n$. Then 
    \[
        \widehat{\mu}(M) := \frac{\sum_{i=1}^M [f(\vx_i) = 1]}{M}
    \]
    is an unbiased estimator for 
    \[
        \mu := \Pr_{\vx \in \{0, 1\}^n}[f(\vx) = 1],
    \]
    and 
    \[
    \Pr[\left|\widehat{\mu}(M) - \mu \right| \leq t] \geq 1 - 2\exp(-2t^2 M),
    \]
    which is at least $1 - \gamma$ for $M = \frac{1}{2t^2} \log(2/\gamma)$.
\end{fact}

As a consequence of the previous idea, although a minimum $\delta$-SR might be hard to compute, this crucially depends on the value of $\delta$. In order to deal with this, our algorithm will sample a value $\delta^\star$ uniformly at random from $[\delta-\varepsilon, \delta+\varepsilon]$, and then compute a minimum $\delta^\star$-SR. Intuitively, the idea is that as $\delta^\star$ is chosen at random, it will be unlikely that a value that makes the computation hard is chosen. 

Before proving~\Cref{prop:smoothed-explanation}, we need to prove a lemma concerning the easiness of selecting the features of the desired explanation.

\subsection{Feature Selection}

Even if we were granted an oracle computing the probabilities
\(
    \Pr_{\vz  \in \D(\vy)}[\M(\vz) = 1]
\), that would not be necessarily enough to efficiently compute a minimum $\delta$-SR. Indeed, for decision trees, the counting problem can be easily solved in polynomial time~\citep{NEURIPS2020_b1adda14}, and yet the computation of $\delta$-SRs of minimum size is hard, even to approximate~\citep{NEURIPS2022_b8963f6a, Kozachinskiy_2023}.
Intuitively, the problem for decision trees is that, even if we were told that the minimum $\delta$-SR has exactly $k$ features, it is not obvious how to search for it better than enumerating all $\binom{d}{k}$ subsets. The case of linear models, however, is different, at least under the uniform distribution. In this case, every feature $i$ that is not part of the explanation will take value $0$ or $1$ independently with probability $\nicefrac{1}{2}$, and contribute to the classification according to its weight $w_i$. In other words, we can sort the features according to their weights (with some care about signs), and select them greedily to build a small $\delta$-SR. A proof for the deterministic case ($\delta = 1$) was already given in~\cite{NEURIPS2020_b1adda14} and sketched earlier on by~\cite{ExplainingNaiveBayes}.

\begin{definition}\label{def:scores}
    Given a linear model $\Lin = (\vw, t)$, and an instance $\vx$, both having dimension $d$, we define the \emph{score} of feature $i \in [d]$ as 
    \[
        s_i := w_i \cdot (2x_i - 1) \cdot (2\Lin(\vx) - 1).    
    \]
\end{definition}

In other words, the sign of $s_i$ is $+1$ if the feature is ``helping'' the classification, and $-1$ if it is ``hurting'' it. The magnitude of $s_i$ is proportional to the weight of the feature $i$. Changing the value of feature $i$ in an instance $\vx$ would decrease $\vw \cdot \vx$ by $s_i$ if $\Lin(\vx) = 1$, and increase it by $s_i$ if $\Lin(\vx) = 0$.
For the uniform distribution (or more generally, any distribution in which all features are Bernoulli variables with the same parameter), we can prove the following lemma that basically states that, for linear models it is good to choose features greedily according to their score.

\begin{lemma}\label{lemma:greedy}
    Given a linear model $\Lin$, and an instance $\vx$, if $\vy^{(0)}, \ldots, \vy^{(d)}$ are the partial instances of $\vx$ such that $\vy^{(k)} \subseteq \vx$ is defined only in the top $k$ features of maximum score, then
    \[ 
        \Pr_{\vz \sim \U(\vy^{(k+1)})}[\Lin(\vz) = \Lin(\vx)] \geq \Pr_{\vz \sim \U(\vy^{(k)})}[\Lin(\vz) = \Lin(\vx)]
    \]
    for all $k \in \{0, \ldots, d-1\}$, and naturally, 
    \[ 
    \Pr_{\vz \sim \U(\vy^{(d)})}[\Lin(\vz) = \Lin(\vx)] = 1.
    \]
    Moreover, $\opt(\Lin, \vx, \delta) = k$ if and only if $\vy^{(k)}$ is a $\delta$-SR for $\vx$, and either $k = 0$ or $\vy^{(k-1)}$ is not a $\delta$-SR for $\vx$. 
    \end{lemma}

Even though a proof of~\Cref{lemma:greedy} is presented in the appendix, let us provide a self-contained example that should convince a reader of the veracity of the lemma.
\begin{example}\label{ex:greedy}
 Consider an instance $\vx = (1,\, 0, \, 0, \, 1, \, 1)$ and the linear model $\Lin$ be defined by 
 \[ 
    \vw = (5, \, 1, \, -3,\, 2, -1) \quad ; \quad t = 5.
 \]
 It is easy to check that $\vw \cdot \vx = 6$, and thus $\Lin(\vx) = 1$. The feature scores, according to~\Cref{def:scores}, are:
 \[
  s_1 = 5, \; s_2 = -1, \; s_3 = 3, \; s_4 = 2, \; s_5 = -1.
 \]
For the first part, the main idea is that a positive score $s_i$ means that the feature is helping the classification (i.e., adding it to a partial instance does not decrease its probability guarantee), while a negative score means that the feature is hurting the classification (i.e., adding it to a partial instance does not increase its probability guarantee). 
Because the partial instances 
\[ \vy^{(0)} \subseteq \vy^{(1)} \subseteq \cdots \vy^{(d)}
\] 
are obtained by adding a single feature at a time, and thus features are added in decreasing order of their scores, then this procedure will have two phases: (i) First, it will add features with a positive score, which raise or maintain the probability of the classification being the same as $\vx$, as the lemma says, and then (ii) it will start adding features with a negative score, which would seem to contradict the lemma, but it turns out that at that point the partial instance $\vy^{(k)}$ would have probability guarantee $1$; this is because $\vy^{(d)} = \vx$, which trivially has probability guarantee $1$.
\Cref{table:ex-greedy1} presents the probabilities associated to the partial instances $\vy^{(0)}, \ldots, \vy^{(d)}$.

For the second part, consider the partial instances $\vy^{\star} = (\bot, \, 0, \, 0, \, 1, \, 1)$ and $\vy^{\dagger} = (1, \, \bot, \, 0, \, 1, \, 1)$. The instance $\vx$ is a completion of both $\vy^{\star}$ and $\vy^{\dagger}$, but $\vy^{\star}$ also has completion 
\[
    \vx^{\star} = (0, \, 0, \, 0, \, 1, \, 1),
\]
whereas $\vy^{\dagger}$ has also completion 
\[
    \vx^{\dagger} = (1, \, 1, \, 0, \, 1, \, 1).
\]
Note that $\vw \cdot  \vx^{\star} = 1 = \vw \cdot \vx - s_1$, whereas $\vw \cdot  \vx^{\dagger} = 6 = \vw \cdot \vx - s_2$. Intuitively, this means that it is better to keep feature $1$ as part of the explanation, but not feature $2$. If we want an explanation with only two features, we should choose feature $1$ and feature $3$, as they have the highest scores. Indeed,~\Cref{table:ex-greedy2} presents the probabilities to all possible explanations of size $2$.
\begin{table}
    \caption{Table of probabilities associated to~\Cref{ex:greedy}.}\label{table:ex-greedy2}
    \centering
    \begin{tabular}{rrr}
        \toprule
        Partial instance & Features included & Probability\\ \midrule 
    $(1, \,\bot, \,0, \,\bot, \,\bot)$ & $\{1, \,3\}$ & $\nicefrac{ 7 }{ 8 }$\\
    $(1, \,\bot, \,\bot, \,1, \,\bot)$ & $\{1, \,4\}$ & $\nicefrac{ 5 }{ 8 }$\\
    $(\bot, \,\bot, \,0, \,1, \,\bot)$ & $\{3, \,4\}$ & $\nicefrac{ 1 }{ 2 }$\\
    $(\bot, \,\bot, \,0, \,\bot, \,1)$ & $\{3, \,5\}$ & $\nicefrac{ 3 }{ 8 }$\\
    $(\bot, \,0, \,0, \,\bot, \,\bot)$ & $\{2, \,3\}$ & $\nicefrac{ 3 }{ 8 }$\\
    $(1, \,\bot, \,\bot, \,\bot, \,1)$ & $\{1, \,5\}$ & $\nicefrac{ 3 }{ 8 }$\\
    $(1, \,0, \,\bot, \,\bot, \,\bot)$ & $\{1, \,2\}$ & $\nicefrac{ 3 }{ 8 }$\\
    $(\bot, \,\bot, \,\bot, \,1, \,1)$ & $\{4, \,5\}$ & $\nicefrac{ 1 }{ 4 }$\\
    $(\bot, \,0, \,\bot, \,1, \,\bot)$ & $\{2, \,4\}$ & $\nicefrac{ 1 }{ 4 }$\\
    $(\bot, \,0, \,\bot, \,\bot, \,1)$ & $\{2, \,5\}$ & $\nicefrac{ 1 }{ 8 }$\\
        \bottomrule
    \end{tabular}
\end{table}
\begin{table}
\caption{Table of probabilities associated to~\Cref{ex:greedy}. The last column denotes the score of the feature added to the partial instance in that row with respect to the previous row.}\label{table:ex-greedy1}
\centering
\begin{tabular}{rrrr}
    Partial instance & Features & Probability & Score \\ \midrule 
$\vy^{(0)}$  & $(\bot, \,\bot, \,\bot, \,\bot, \,\bot)$ & $\nicefrac{ 1 }{ 4 }$ & - \\
$\vy^{(1)}$  &$(1, \,\bot, \,\bot, \,\bot, \,\bot)$ & $\nicefrac{ 1 }{ 2 }$ & 5\\
$\vy^{(2)}$  &$(1, \,\bot, \,0, \,\bot, \,\bot)$ & $\nicefrac{ 7 }{ 8 }$ & 3\\
$\vy^{(3)}$ & $(1, \,\bot, \,0, \,1, \,\bot)$ & $\nicefrac{ 1 }{ 1 }$ & 2 \\
$\vy^{(4)}$ &$(1, \,0, \,0, \,1, \,\bot)$ & $\nicefrac{ 1 }{ 1 }$ & -1 \\
$\vy^{(5)}$ & $(1, \,0, \,0, \,1, \,1)$ & $\nicefrac{ 1 }{ 1 }$ & -1\\
\bottomrule
\end{tabular}
\end{table}
\end{example}


With~\Cref{lemma:greedy} in hand, we can proceed to prove~\Cref{prop:smoothed-explanation}.

\input{sections/proof-of-thm1}



%% file: sections/proof-of-thm1.tex
\begin{algorithm}[tb]
	\caption{LinearMonteCarloExplainer}
	\label{alg:algorithm}
	\textbf{Input}: A linear model $\Lin$, an instance $\vx$, and $\delta \in (0, 1)$.\\
	\textbf{Parameters}:  $\varepsilon \in (0, 1)$,  $\gamma \in (0, 1)$.\\
	\textbf{Output}: A value $\delta^\star \in [\delta-\varepsilon, \delta+\varepsilon]$ together with a minimum $\delta^\star$-SR explanation for $\vx$.\\
	\begin{algorithmic}[1] 
	\STATE $\delta^\star \gets$ uniformly random sample from $[\delta-\varepsilon, \delta + \varepsilon]$ \label{line:delta}
	\FOR{$i \in \{1, \ldots, d\}$}
		\STATE $s_i \gets  w_i \cdot (2x_i - 1) \cdot (2\Lin(\vx) - 1)$
	\ENDFOR
	
	\FOR{$k \in\{0,1 \ldots, d\}$}
		\STATE  let $\vy^{(k)} \subseteq \vx$ be the partial instance defined only in the top $k$ features with maximum score $s_i$.\label{line:greedy}
	\ENDFOR
	\STATE $M \gets (\log^2{d})/(2\varepsilon^2 \gamma^2) \log(2 \log d / \gamma)$\label{line:mdef}
	\STATE $\textrm{LB} \gets 0$, $\textrm{UB} \gets d$, and $\textsc{steps} \gets 0$

	\WHILE{$\textrm{LB} \neq \textrm{UB}$ and $\textsc{steps} \leq \log d$}\label{line:while}
		\STATE $\textsc{steps} \gets \textsc{steps} + 1$\label{line:steps}
		\STATE $m \gets \left(\textrm{LB} + \textrm{UB} \right)/2$
		\STATE $\widehat{v}_m \gets \textsc{MonteCarloEstimation}(\Lin, \vy^{(m)}, \vx, M)$ \label{line:montecarlo}
		\IF {$\widehat{v}_m \geq \delta^\star$}
			\STATE $\textrm{UB} \gets m$
		\ELSE
			\STATE $\textrm{LB} \gets (m+1)$
		\ENDIF
	\ENDWHILE\label{line:endwhile}
	\STATE $k^\star \gets \textrm{LB}$ (or equivalently, $\textrm{UB}$)
	 \RETURN $(\delta^\star, \vy^{(k^\star)})$
	\end{algorithmic}
	\end{algorithm}

\begin{algorithm}[tb]
	\caption{MonteCarloEstimation}	\label{alg:montecarlo}
	\textbf{Input}: A linear model $\Lin$, a partial instance $\vy$, an instance $\vx$, and a number of samples $M \in \mathbb{N}$.\\
	\textbf{Output}: An estimate~$\widehat{v}$~of~$\Pr_{\vz \sim \U(\vy)}[\Lin(\vz)=\Lin(\vx)]$.\\
	\begin{algorithmic}[1]
	\STATE $\widehat{v} \gets 0$
	\FOR{$i = 1$ to $M$}
		\STATE Sample $\vz \sim \U(\vy)$
		\IF {$\Lin(\vz) = \Lin(\vx)$}
			\STATE $\widehat{v} \gets \widehat{v} + 1$
		\ENDIF \ENDFOR
	 \RETURN $\widehat{v}/M$.
\end{algorithmic}
\end{algorithm}

\begin{proof}[Proof of~\Cref{prop:smoothed-explanation}]
We use~\Cref{alg:algorithm}. Let us define the partial instances $\vy^{(0)}, \vy^{(1)} \ldots, \vy^{(d)}$ so that $\vy^{(k)} \subseteq \vx$ is the partial instance defined only in the $k$ features with maximum score (line~\ref{line:greedy}). 
We then define a sequence of values
  $v_k$ as
	\[
		v_k := \Pr_{\vz \sim \U(\vy^{(k)})}[\Lin(\vz) = \Lin(\vx)],
	\]
	and note that due to~\Cref{lemma:greedy}, the sequence $v_0, \ldots, v_d$ is non-decreasing.
Let $M = \frac{\log^2{d}}{2\varepsilon^2 \gamma^2} \log(2 \log d / \gamma)$, as in line~\ref{line:mdef}, and let us define random variables $\widetilde{v_k}$ as follows: if~\Cref{alg:algorithm} enters line~\ref{line:montecarlo} with $m=k$, then $\widetilde{v_k}$ is the output of~\Cref{alg:montecarlo} (i.e., $\widehat{v}_k(M)$), and otherwise $\widetilde{v_k} = v_k$.  
We use binary search (lines~\ref{line:while}-\ref{line:endwhile}), to find $k^\star$, the smallest $k$ such that
\[
	\widetilde{v_k} \geq \delta^\star,
\]
and our goal is to show that with good probability $k^\star$ is also the smallest $k$ such that $v_k \geq \delta^\star$, which would imply the correctness of the algorithm by~\Cref{lemma:greedy}.
Note, however, that even though the sequence $v_0, \ldots, v_d$ is non-decreasing (\Cref{lemma:greedy}), the estimated values $\widehat{v_k}$ are not necessarily so.
Let $S$ be a random variable corresponding to the set of values $k$ such that~\Cref{alg:algorithm} enters line~\ref{line:montecarlo} with $m=k$, and note that if for every $k$ in $S$ it happens that the events 
\[
	A_k := \left(v_k \geq \delta^\star \right) \text{ and }  B_k := \left(\widetilde{v_k} \geq \delta^\star\right)
\]
are equivalent (i.e., either both occur or neither occurs), then the algorithm will succeed, as that would indeed imply that $k^\star$ is the smallest $k$ such that $v_k \geq \delta^\star$. 

Then, for $k \in [d]$, define events $E_k$ and $F_k$ as follows:
\begin{align*}
	E_k &:= |\delta^\star - v_k| \geq \frac{\varepsilon \gamma}{\log d},\\
	F_k &:= |\widetilde{v_k} - v_k| \leq \frac{\varepsilon \gamma}{\log d}.
\end{align*}
We claim that if both $E_k$ and $F_k$ hold for some $k$, then $A_k$ and $B_k$ are equivalent events for that $k$. Indeed,
\begin{align*}
	A_k &\iff v_k \geq \delta^\star\\
		&\iff v_k \geq \delta^\star + \frac{\varepsilon \gamma}{\log d} \tag{by $E_k$}\\
		&\iff v_k - \frac{\varepsilon \gamma}{\log d} \geq \delta^\star\\
		&\iff \widetilde{v_k} \geq \delta^\star \tag{by $F_k$}\\
		&\iff B_k.
\end{align*}

Thus, if we show that $E_k$ and $F_k$ hold with good probability for every $k \in S$, we can conclude the theorem.
First, note that because of the condition on the variable $\textsc{steps}$ (lines~\ref{line:while}, \ref{line:steps}) we have $|S| \leq \log d$, allowing us to do a binary search in case the desired events $E_k$ and $F_k$ hold, and preventing the algorithm from looping otherwise; this way the runtime is bounded not only on expectation but deterministically. \footnote{For simplicity, we will say $|S| \leq \log d$, even though the exact bound for a binary search is $|S| \leq \lfloor \log d  + 1 \rfloor$. This choice naturally has no impact on the asymptotic analysis of the algorithm.} Then, note that for any $k$ we have
\begin{align*}
	\Pr\left[ \, \overline{F_k} \,\right] \leq \Pr\left[ \, \overline{F_k} \mid k \in S \right] &= \Pr\left[|\widehat{v_k}(M) - v_k| > \frac{\varepsilon \gamma}{\log d}\right]\\
	&\leq \frac{\gamma}{\log d}. \tag{by~\Cref{fact:hoeffding}}
\end{align*}
Because $S$ itself is a random variable, whose size is also a random variable, we need to be careful before applying a union bound or any related tricks. Let us refer to the elements of $S$ as $\{s_1, \ldots, s_\ell\}$, and let us call $F(i)$, for $i \in [\log d]$,\footnote{Throughout this proof, we use notation $[\alpha]$, for $\alpha \in \mathbb{R}^{> 0}$, to denote the set $\{0, 1, \ldots, \lceil \alpha \rceil - 1, \lceil \alpha \rceil\}$.} to the event $F_{s_i}$ if $i \leq \ell$, and to the sample space $\Omega$ (i.e., the event that always happens) otherwise. Then, we claim that for any $0 \leq i \neq j \leq \lceil\log d\rceil$, we have 
\begin{equation}\label{eq:quasi-independence}
	\Pr[F(i) \cap F(j)] = \Pr[F(i)] \cdot \Pr[F(j)],
\end{equation}
as either $\max \{i, j\} \leq \ell$, in which case the claim holds by independence (since both events depend only on disjoint sets of independent random samples), or the claim holds trivially since $\Pr[F(i)] = 1$ for $i > \ell$.
Therefore, we have 
\begin{align*}
\Pr\left[\bigcap_{k \in S} F_k\right] &= \Pr\left[F(0) \cap F(1) \cap \cdots \cap F(\lceil\log d\rceil)\right]\\
&= \prod_{i \in [\log d]} \Pr[F(i)]\tag{by~\Cref{eq:quasi-independence}}\\
&\geq \left(1 - \frac{\gamma}{\log d}\right)^{\log d} \geq 1 - \gamma.
\end{align*}
We now argue that the event $\bigcap_{k \in S} E_k$ happens with good probability.
To see that, note first that for every $k \in [d]$, line~\ref{line:delta} implies
\[
	\Pr[\overline{E_k}] = \Pr\left[\delta^\star \in \left[v_k \pm \frac{\varepsilon \gamma}{\log d}\right]\right] \leq \frac{\frac{2\varepsilon \gamma}{\log d}}{2\varepsilon} = \frac{\gamma}{\log d}.
\]
Once again, we need to be careful as the events $E_k$ are not independent of $S$, nor between them this time.
Using the law of total probabilities, we have 
\begin{align*}
	\Pr\left[\bigcap_{k \in S} E_k\right] &= \sum_{S' \subseteq [d]} \Pr\left[S = S'  \mid \bigcap_{k \in S'} E_k \right] \Pr\left[\bigcap_{k \in S'} E_k\right]\\
	&= \sum_{\substack{S' \subseteq [d]\\ |S'| \leq \log d}} \Pr\left[S = S' \mid \bigcap_{k \in S'} E_k\right] \Pr\left[\bigcap_{k \in S'} E_k\right],
\end{align*}
where we can now effectively use the union bound to say that for any fixed $S'$ with $|S'| \leq \log d$ we have
\[ 
	\Pr\left[\bigcap_{k \in S'} E_k \right] \geq 1 - \gamma.
\]
Therefore, we conclude that
\begin{equation}
	\Pr\left[\bigcap_{k \in S} E_k\right] =  \sum_{\substack{S' \subseteq [d]\\ |S'| \leq \log d}} \Pr\left[S = S'\mid \bigcap_{k \in S'} E_k\right] \Pr\left[\bigcap_{k \in S'} E_k\right] \geq (1-\gamma) \sum_{\substack{S' \subseteq [d]\\ |S'| \leq \log d}} \Pr\left[S = S' \mid \bigcap_{k \in S'} E_k\right]. \label{eq:align}
\end{equation}
It is not obvious, however, whether the sum on the right-hand side of~\Cref{eq:align}~equals $1$; we will, however, argue that it is at least $1-\gamma$. 
Recall that $\Pr[F_k | k \not\in S] \geq \Pr[F_k | k \in S]$ for any $k$, from where it follows that for every set $S'$ of size at most $\log d$ we have 
\[ 
	\Pr\left[\bigcap_{k \in S'} F_k \right] \geq \Pr\left[\bigcap_{k \in S} F_k \right].	
\]
Then, note that for any index $k \in [d]$, the event $F_k$ is conditionally independent of all events $E_{j}$, with $j \in [d]$ given the event $k \in S$. That is,
\[
	\Pr\left[F_k | k \in S\right] = \Pr\left[F_k | E_{j}, k \in S\right], \quad \text{for any } j \in [d].
\]
We thus deduce that for any fixed $S'$ with $|S'| \leq \log d$ we have
\begin{align}\label{eq:2}
	\begin{split}
	\Pr\left[\bigcap_{k \in S'} F_k \mid \bigcap_{k \in S'} E_k\right] &    \geq  \Pr\left[\bigcap_{k \in S'} F_k \mid \bigcap_{k \in S'} E_k \text{ and  } \bigcap_{k \in S'} {k \in S}\right] \\
	&= \Pr\left[\bigcap_{k \in S'} F_k \mid \bigcap_{k \in S'} {k \in S}\right] \geq \Pr\left[\bigcap_{k \in S} F_k\right] \geq (1-\gamma).
	\end{split}
\end{align}
Then, our key observation is that there is a single value $S^\star \subseteq [d]$, with $|S^\star| \leq \log d$, that the binary search can take if we condition on all the events $E_k$ and $F_k$ happening, since in that case events $A_k$ and $B_k$ coincide. In other words, there exists $S^\star$, with $S^\star \subseteq [d]$ and $|S^\star| \leq \log d$, such that
\begin{equation}\label{eq:prob1}
\Pr\left[S = S^\star \mid \bigcap_{k \in S^\star} E_k \cap \bigcap_{k \in S^\star} F_k\right] = 1.
\end{equation}
We can then argue as follows:
\begin{align*}
	\Pr\left[\bigcap_{k \in S} E_k\right] &\geq (1-\gamma) \sum_{\substack{S' \subseteq [d]\\ |S'| \leq \log d}} \Pr\left[S = S' \mid \bigcap_{k \in S'} E_k\right]\tag{by~\Cref{eq:align}}\\
	&\geq (1-\gamma) \Pr\left[S = S^\star  \mid \bigcap_{k \in S^\star} E_k \right]\\
	&\geq (1-\gamma) \Pr\left[S = S^\star \text{ and } \bigcap_{k \in S^\star} F_k \mid \bigcap_{k \in S^\star} E_k \right]\\
	&= (1-\gamma) \Pr\left[S = S^\star \mid \bigcap_{k \in S^\star} E_k \cap \bigcap_{k \in S^\star} F_k\right] \cdot \Pr\left[\bigcap_{k \in S^\star} F_k \mid \bigcap_{k \in S^\star} E_k\right]\tag{Bayes' rule}\\
	&\geq (1-\gamma)^2. \tag{by~\Cref{eq:2,eq:prob1}}
\end{align*}

Therefore, the algorithm will succeed with probability at least 
\[ 
	\Pr\left[\bigcap_{k \in S} E_k\right] \cdot \Pr\left[\bigcap_{k \in S} F_k\right] \geq (1-\gamma)^3 \geq 1-3\gamma.
\]
The runtime is simply
$O(\log d \cdot M \cdot d  )$; as (i) the binary search performs $O(\log d)$ steps; (ii) each of the binary search steps requires $M$ samples, and (iii) each sample requires evaluating the model $\Lin$ and thus takes time $O(d)$. Naturally, running the algorithm with $\gamma' = 1/3 \cdot \gamma$ will yield a success probability of $1-\gamma$ without changing the asymptotic runtime, and thus we conclude the proof.

\end{proof}

%% file: sections/locally-minimal.tex
Due to the complexity of finding even subset-minimal $\delta$-SR,~\citet{izza2024locallyminimalprobabilisticexplanations} have proposed to study ``locally minimal'' $\delta$-SR, which are $\delta$-SRs such that the removal of any feature from the explanation would decrease its probabilistic guarantee below $\delta$.
Interestingly, we can generalize a proof from~\cite{NEURIPS2022_b8963f6a} to show that, over lineal models even in the more general case of product distributions (distributions over $\{0,1\}^d$ that are products of independent Bernoulli variables of potentially different parameters), every locally minimal $\delta$-SR is a subset-minimal $\delta$-SR. This allows leveraging the previous results of~\citet{izza2024locallyminimalprobabilisticexplanations} to subset-minimal $\delta$-SRs in the case of linear models.

\begin{theorem}\label{thm:locally-minimal}
For linear models, under any product distribution, every locally minimal $\delta$-SR is a subset-minimal $\delta$-SR.
\end{theorem}
\begin{proof}[Proof sketch]
    Define the \emph{``locality''} gap $\textsc{lgap}(\vy)$ of a locally minimal $\delta$-SR $\vy$ as the smallest value $g$ such that $|\vy^\star|_\bot - |\vy|_\bot = g$ for some $\vy^\star \subseteq \vy$ that is a $\delta$-SR.
    If $g = 0$, then $\vy$ is globally minimal, and we are done. If $g$ were to be $1$, then $\vy$ would not be locally minimal, a contradiction. Therefore, we can safely assume $g \geq 2$ from now on.
    Let $\Lin, \vy$ be such that $\vy$ is locally minimal $\delta$-SR and $\textsc{lgap}(\vy) \geq 2$. We will find a contradiction by the following method:
    \begin{itemize}
        \item Let $\vy^\star$ be the $\delta$-SR such that $| \vy \setminus \vy^\star | = \textsc{lgap}(\vy)$.
        \item Every feature in $\vy \setminus \vy^\star$ is either ``good'', if its score is positive, or ``bad'' if its score is negative.
        \item Fix any feature $i$ in $\vy \setminus \vy^\star$. If $i$ is good, then $\vy^\star \oplus i$, meaning the partial instance obtained by taking $\vy$ and setting its $i$-th feature to $x_i$, has a probability guarantee greater or equal than that of $\vy^\star$ (the proof of this fact is very similar to the proof of~\Cref{lemma:greedy}), and the gap has reduced.
                On the other hand, if $i$ is bad, then $\vy \ominus i$, meaning the partial instance obtained from $\vy$ by setting $y_i = \bot$, has greater-equal probability than $\vy$, contradicting the fact that $\vy$ is locally minimal.
    \end{itemize}

\end{proof}

%% file: sections/size-approximations.tex
A natural question at this point is whether the size of a $(\delta, \varepsilon)$-min-SR is necessarily similar to the size of a $(\delta, 0)$-min-SR (i.e., a smallest $\delta$-SR). It turns out that this is not the case, and it can happen that in order to get a slightly better probabilistic guarantee (i.e., $\delta + \varepsilon$ instead of $\delta$), the number of features needed under any explanation significantly increase.
In general, if we let $\opt(\M, \vx, \delta)$ denote the size of the smallest $\delta$-SR for $(\M, \vx)$, we can prove the following  generalization of~\Cref{ex:delta-sr-size}.

\begin{proposition}\label{prop:delta-sr-size}
For any $\delta \in (0, 1)$, $\gamma > 0$, and any $\varepsilon > 0$ such that $\delta + \varepsilon \leq 1$, there are pairs $(\Lin, \vx)$ where $\Lin$ is a linear model of dimension $d$, and $\vx$ an instance of dimension $d$, such that
	\[ 
		\frac{\opt(\Lin, \vx, \delta+\varepsilon)}{\opt(\Lin, \vx, \delta)} = \Omega\left(d^{\frac{1}{2} - \gamma}\right).
	\]
\end{proposition}

As a consequence, we may say informally that approximations on $\delta$ do not neccesarily lead to approximations on the explanation size.

%% file: sections/conclusions.tex
We have proved a positive result for the case of linear models, showing that a $(\delta, \varepsilon)$-min-SRs can be computed efficiently, and also a more abstract reason suggesting that linear models might be easier to explain than, e.g., decision trees. However,  a variety of natural questions and directions of research remain open. 
First, in practical terms, even though the runtime of~\Cref{prop:smoothed-explanation} is polynomial and only has a quasi-linear dependency on $d$, our future work includes lowering the dependency in $1/\varepsilon$ and $1/\gamma$; on a dataset with $d = 500$, setting $\varepsilon = 0.1$ and $\gamma = 0.01$ is already computationally expensive. 
We acknowledge, in terms of practical implementations, the work of~\citet{Louenas,izza2024locallyminimalprobabilisticexplanations} that allows for computing small probabilistic explanations over decision trees significantly faster than the exact SAT approach of~\citet{NEURIPS2022_b8963f6a}.
Similarly,~\citet{izzaComputingProbabilisticAbductive2023} showed solid practical results with different kinds of classifiers, including linear models (i.e., Naive Bayes). Despite our results having better theoretical guarantees over linear models, a natural direction of future work is to improve the practical efficiency of our algorithm for high-dimensional models.

Second, our theoretical result has some natural directions for generalization. We considered only binary features, whereas in order to offer a practically useful tool to the community, we will need to understand how to compute (approximate) probabilistic explanations for mixtures real-valued features and categorical features, for example under the ``extended linear classifier'' definition of~\citet{DBLP:conf/nips/0001GCIN20}.  Another fascinating theoretical question is handling the generalization of our setting to that of product distributions (i.e., feature $i$ takes value $1$ with probability $p_i$ and $0$ otherwise) can also be solved efficiently. A straightforward extension of our techniques does not seem to work on such a generalized setting, since the \emph{feature selection} argument of~\Cref{subsec:feature_selection} no longer holds. Therefore, we believe that new techniques will be needed.  
 
Third, it would be interesting to allow for a more declarative way of specifying the probabilistic guarantees or constraints on the explanations. While a recent line of research has studied the design of languages for defining explainability queries with a uniform algorithmic treatment~\citep{arenasFoundationsSymbolicLanguages2021,bps2020,KR2024-6}, we are not aware of any work on that line that allows for probabilistic terms.

%% file: sections/appendix.tex
\subsection{\texorpdfstring{Calculation for~\Cref{ex:delta-sr-size}}{Calculation for Example 1}}
\label{subsec:ex1}
Consider the following version of Chernoff bound.
\begin{lemma}[Chernoff bound]
    Let $X$ be a finite sum of independent Bernoulli variables, with $\E[X] = \mu > 0$. Then for any $t \geq 0$ we have
    \[
    \Pr \Big[ \left|X - \mu\right| \geq t \Big] \leq 2\exp\left(\frac{-t^2}{3 \mu} \right).
    \] 
    
    \label{lemma:chernoff}  
    \end{lemma}
If we define the Bernoulli variables $X_i := (z_i = 1)$ for $\vz \sim \U(\vy)$ (the uniform distribution over $\comp(\vy)$), then the variables $X_2, \ldots, X_{1000}$ are identical independent Bernoulli variables, and if $X = \sum_{i = 2}^{1000} X_i$ we have 
\[
    \mu := \mathbb{E}[X] =  999\mathbb{E}[X_2] = \frac{999}{2},
\]
as each $X_i$ has expectation $\frac{1}{2}$ becuase $\U$ is the uniform distribution. Then, using \Cref{lemma:chernoff}, we have that 
\begin{align*}
    \Pr\left[ X < 250 \right] &\leq 2 \exp\left(\frac{-(445-250)^2    \cdot 2}{3 \cdot 999}\right)\\ &\approx 2 \exp(-32.29...)\\ &< 2 \cdot 10^{14}.
\end{align*}
To conclude, note that
\begin{align*}
    \Pr_{\vz \sim \U(\vy)}
     \left[\Lin(\vz) = 1 \right] 
      &= \Pr_{\vz \sim \U(\vy)}
    \left[\vz \cdot \vw \geq 1250 \right]\\
    &= \Pr_{\vz \sim \U(\vy)}\left[1000 + \sum_{i=2}^{1000}{z_i} \geq 1250\right]\\
    &= 1 - \Pr\left[\sum_{i=2}^{1000}{z_i} < 250\right]\\
    &= 1 - \Pr\left[ X < 250 \right] > 1 - 2\cdot 10^{14} > 0.999999.
\end{align*}
    
\subsection{\texorpdfstring{\Cref{prop:delta-sr-size}}{Proposition 1}}

In order to prove~\Cref{prop:delta-sr-size} we will need some auxiliary lemmas concerning binomial distributions. The only previous result we will need is a simple upper bound on the central binomial coefficient, attributed to Erd\H{o}s~\citep{133742}.
\begin{lemma}[Erd\H{o}s' bound (see~\citep{133742})]
    For every $n \geq 1$ we have
    \[
    \binom{2n}{n} \leq \frac{4^n}{\sqrt{2n+1}}.
    \]
    \label{lemma:erdos}
\end{lemma}
We can now proceed to prove our auxiliary lemmas.
\begin{lemma}\label{lemma:p-diff}
Let $P(n, k)$ be the probability that a binomial random variable $X \sim \text{Bin}(n, 1/2)$ is at least $k$, and define 
\[
    g(n) = \max_{k \in \{1, \ldots, n\}} P(n-1, k-1) - P(n, k).
\]
Then, $0 \leq g(n) \leq 1/\sqrt{n}$ for every $n \geq 2$.
\end{lemma}
\begin{proof}
We can think of $P(n, k)$ as the probability of getting at least $k$ heads out of $n$ tosses of a fair coin. By casing on the value (heads or tails) of the first coin we can see that for $n \geq k \geq 1$ we have
\[ 
    P(n, k) = \frac{1}{2}P(n-1, k-1) + \frac{1}{2}P(n-1, k),
\]
which we can rearrange into 
\[ 
    P(n, k) = P(n-1, k-1) + \frac{1}{2}\big[P(n-1, k) - P(n-1, k-1)\big],
\]
and thus
\[ 
    P(n-1, k-1) - P(n, k) = \frac{1}{2}\big[P(n-1, k-1) - P(n-1, k)\big].
\]
But $P(n-1, k-1) - P(n-1, k)$ is the probability of getting exactly $k-1$ heads in $n-1$ tosses, which is $\frac{1}{2^{n-1}} \binom{n-1}{k-1}$, and clearly positive. Hence, using~\Cref{lemma:erdos}, we have
\[ 
    0 \leq P(n-1, k-1) - P(n, k) \leq \frac{1}{2^{n}} \binom{n-1}{k-1} \leq \frac{1}{2^{n}} \binom{n-1}{\lfloor (n-1)/2 \rfloor} \leq \frac{1}{2^{n}} \frac{2^n}{\sqrt{ 2n - 1}} \leq \frac{1}{\sqrt{n}}. 
\]
\end{proof}

\begin{lemma}[Binomial Approximation]\label{lemma:binomial-approx} 
    Let $P(n, k)$ be the probability that a binomial random variable $X \sim \text{Bin}(n, 1/2)$ is at least $k$. Then, given a probability $\delta \in (0, 1)$ and a targer error $\varepsilon > 0$, there exist values $n \leq \varepsilon^{-2}$ and $k \leq n$ such that 
    \[ 
        |P(n, k) - \delta| < \varepsilon.
    \]
\end{lemma}
\begin{proof}
Fix any value of $n$     and note that $P(n, k-1) - P(n, k)$ is the probability of getting exactly $k-1$ heads in $n$ tosses of a fair coin. Thus, for any $k \in \{1, \ldots, n\}$ we can use~\Cref{lemma:erdos} to say
\begin{equation}\label{eq:prob-diff}
    P(n, k-1) - P(n, k) = \frac{1}{2^{n}} \binom{n}{k-1} \leq \frac{1}{2^{n}} \binom{n}{\lfloor n/2 \rfloor} \leq \frac{1}{2^{n}} \frac{2^n}{\sqrt{2n+1}} \leq \frac{1}{\sqrt{n}}.
\end{equation}
Now, choose a large enough $n$ so that $\frac{1}{\sqrt{n}} < \varepsilon$. Then, choose $k$ as the largest integer such that $P(n, k) \geq \delta$, which exists since $P(n, 0) = 1$ and $P(n, n+1) = 0$. This definition of $k$ yields 
\[ 
    P(n, k) \geq \delta > P(n, k+1),
\]
and using~\Cref{eq:prob-diff} we can see that 
\[ 
    |P(n, k+1) - \delta| = \delta - P(n, k+1) \leq P(n, k) - P(n, k+1) \leq \frac{1}{\sqrt{n}} < \varepsilon,
\]
which completes the proof.
\end{proof}

We are finally ready to prove~\Cref{prop:delta-sr-size}.
\begin{proof}[Proof of~\Cref{prop:delta-sr-size}]
Let $\delta$ and $\varepsilon$ be the values in the statement of the proposition. 
Let $n$ and $m$ be large integers whose precise value will be determined later. We will construct a linear model $\Lin$ of dimension $d = n+1$, with threshold $t = 2$, and weights $w_1 = 1$, $w_i = 1/m$ for $i \in \{2, \ldots, d\}$. 
Let $\vx = (1, 1, \ldots, 1) \in \{0, 1\}^d$ be the instance to explain, and $\vy = (1, \bot, \bot, \ldots, \bot) \subseteq \vx$ be a partial instance.
Now, use~\Cref{lemma:binomial-approx} with parameters $\delta' = \delta + \varepsilon/4$ and $\varepsilon' = \varepsilon/4$ to choose values for $n, m$ such that 
\[ \left|P\big(n,\, m\big) - \delta'\right| < \varepsilon', \]
 with $m \leq n$. Note now that $\Lin(\vx) = 1$, and that
 \begin{equation}\label{eq:prob-lin}
    \Pr_{\vz \sim \U(\vy)}\left[\Lin(\vz) = 1\right] = P\big(n,\, m\big),
 \end{equation}
 since in order for a completion of $\vy$ to have $\Lin(\vz) = 1$ it must have at least $m$ features (besides the 1st one) set to $1$, as that way its value according to the weights $w_i$ will be at least $1 + m \cdot 1/m = 2 = t$.
 As we have 
 \[ 
    \left|P\big(n,\, m\big) - (\delta+ \varepsilon/4) \right| < \varepsilon/4, 
 \]
 we can deduce 
 \[ 
    \delta \leq P\big(n,\, m\big) \leq \delta + \varepsilon/2.
 \]
 Combining our last equation with~\Cref{eq:prob-lin} we have that $\vy$ is a $\delta$-SR for $\vx$ of size $1$.
 Now, let us see how many more features need to be added to $\vy$ in order to get a $(\delta+\varepsilon)$-SR.
 Suppose that we add $k$ features to $\vy$ (as all features are equal, we add any $k$ of them) obtaining a partial instance $\vy'$. Then, we have that 
 \[ 
    \Pr_{\vz \sim \U(\vy')}\left[\Lin(\vz) = 1\right] = P\big(n - k,\, m - k\big),
 \]
 as now $k$ features have been set to $1$. We can now use~\Cref{lemma:p-diff} $k$ times in a row:
 \begin{align*}
    P(n-k, m - k) &\leq P(n-k+1, m - k+1) + \frac{1}{\sqrt{n-k}}\\
    &\leq P(n-k+2, m - k+2) + \frac{1}{\sqrt{n-k}} + \frac{1}{\sqrt{n-k+1}}\\
    & \leq P(n-k+2, m - k+2) + \frac{2}{\sqrt{n-k}}\\
    & \vdots\\
    & \leq P(n, m) + \frac{k}{\sqrt{n-k}}.
 \end{align*}

 Therefore, we have that 
 \[ 
 P(n-k, m - k) \leq P(n, m) + \frac{k}{\sqrt{n-k}} \leq \delta + \varepsilon/2 + \frac{k}{\sqrt{n-k}}.
 \]
 It now remains to show that, unless $k$ is large enough (relatively to $n$), the right-hand side of the previous inequality is less than $\delta + \varepsilon$.
 Indeed, take $k = n^{1/2 - \gamma}$ for any positive $\gamma$. Then, we want to show that for sufficiently large $n$ we have
 \[ 
 \frac{k}{\sqrt{n-k}} < \varepsilon/2.
 \]
To do so, note that
\[
\frac{k}{\sqrt{n-k}} < \frac{k}{\sqrt{n}}
                     = \frac{n^{1/2 - \gamma}}{n^{1/2}}
                     = n^{-\gamma},
\]
from where taking $n > \left(\frac{2}{\varepsilon}\right)^{1/\gamma}$ suffices to complete the proof.
\end{proof}

\subsection{\texorpdfstring{\Cref{prop:hardness}}{Proposition 2}}
\begin{proof}[Proof of~\Cref{prop:hardness}]
    The proof is a twist on ~\citet[Lemma 28]{NEURIPS2020_b1adda14}; let $(s_1, \ldots, s_n, T) \in \mathbb{N}^{n+1}$ be an instance of the $\# \ptime$-complete problem $\#\Knapsack$, that consists on counting the number of sets $S \subseteq \{s_1, \ldots, s_n\}$ such that $\sum_{s \in S}s \leq T$.  
    We can assume that $\sum_{i=1}^n s_i > T$, as otherwise the $\# \Knapsack$ instance is trivial.
    Then, let $\Lin$ be a linear model with weights $w_i = s_i$, and threshold $t = T+1$.
    Now, consider the problem of deciding whether $\#\Knapsack(s_1, \ldots, s_n, T) \geq m$ for an input $m$, which cannot be solved in polynomial time unless $\ptime = \# \ptime$.
    Let $\vx = (1, 1, \ldots, 1)$, and $\delta = \frac{m}{2^{n}}$. We claim that $(\Lin, \vx, \delta, k=0)$ is a Yes-instance for Uniform-Min-$\delta$-SR$(\textsc{Linear})$ if and only if $\#\Knapsack(s_1, \ldots, s_n, T) \geq m$. 
    First, note that $\Lin(\vx) = 1$, since $\sum_{i=1}^n w_i x_i = \sum_{i=1}^n s_i \geq T+1 = t$. 
    For every set $S \subseteq \{s_1, \ldots, s_n\}$ such that $\sum_{s \in S}s \leq T$, its complement $\overline{S} := \{s_1, \ldots, s_n\} \setminus S$ holds $\sum_{s \in \overline{S}}s > T$, and as all values are integers, this implies as well
    \[
        \sum_{s \in \overline{S}} s \geq T+1 = t.
    \]
    To each such set $\overline{S}$, we associate the instance $\vz(\overline{S})$ defined as
    \[
        z(\overline{S})_i = \begin{cases}
            1 & \text{if } s_i \in \overline{S}\\
            0 & \text{otherwise}.
        \end{cases}
    \]
    Now note that 
    \[
        \Lin(\vz(\overline{S})) = \begin{cases}
            1 & \text{if } \sum_{s_i \in \overline{S}} w_i \geq T+1\\
            0 & \text{otherwise} \end{cases} = 1,
    \]
    and thus there is a bijection between the sets $S$ whose sum is at most $T$ and the instances $\vz(\overline{S})$ such that $\Lin(\vz(\overline{S})) = 1 = \Lin(\vx)$.
    To conclude, simply note that the previous bijection implies
    \[
        \Pr_{\vz \sim \U(\bot^d)}\Big[\Lin(\vz) = 1\Big] = \frac{\#\Knapsack(s_1, \ldots, s_n, T)}{2^n},
    \] 
    and thus the ``empty explanation'' $\bot^d := (\bot, \bot, \ldots, \bot)$ has probability at least $\delta$ if and only if 
    \[ 
    \#\Knapsack(s_1, \ldots, s_ n, T) \geq m.
    \]  As the empty explanation is the only one with size $\leq 0$, we conclude the proof.

\end{proof}

\subsection{Lemma 1 and Theorem 4}
Let us state the result about locally minimal $\delta$-SRs as a theorem, so we can proceed to prove it. 
In order to state it, however, we need to define what we mean by an arbitrary product distribution. Consider Bernoulli variables $X_1, \ldots, X_d$ with probabilities $p_1, \ldots, p_d$ respectively, and let us denote by $\D = X_1 \times \cdots \times X_d$ their joint distribution, which we will call a product distribution. Then, $\D$ induces a distribution over the set $\{0, 1\}^d$ by the rule
\[ 
    \Pr_{\vz \sim \D}[\vz = \vx] = \prod_{i=1}^d p_i^{x_i} (1-p_i)^{1-x_i}, \quad \forall x \in \{0, 1\}^d.
\]
Naturally, this induces a distribution over $\comp(\vy)$ for any partial instance $\vy$ as follows:
\[ 
    \Pr_{\vz \sim \D(\vy)}[\vz = \vx] =
        \frac{\Pr_{\vz \sim \D}[\vz = \vx]}{\sum_{\vw \in \comp(\vy)} \Pr_{\vz \sim \D}[\vz = \vx]} \quad \forall \vx \in \comp(\vy),
\]
and naturally $\Pr_{\vz \sim \D(\vy)}[\vz = \vx] = 0$ if $\vx \not\in \comp(\vy)$.
With this notation, we can make a precise theorem statement.
\begin{theorem}
    For linear models, under any product distribution, every locally minimal $\delta$-SR is a subset-minimal $\delta$-SR.
\end{theorem}

Furthermore, we will prove a stronger lemma than Lemma 1 stated in the main body.

\newtheorem*{theorem8}{Lemma 1}
\begin{theorem8} 
    Given a linear model $\Lin$, and an instance $\vx$, if $\vy^{(0)}, \ldots, \vy^{(d)}$ are the partial instances of $\vx$ such that $\vy^{(k)} \subseteq \vx$ is defined only in the top $k$ features of maximum score, then
    \[ 
        \Pr_{\vz \sim \U(\vy^{(k+1)})}[\Lin(\vz) = \Lin(\vx)] \geq \Pr_{\vz \sim \U(\vy^{(k)})}[\Lin(\vz) = \Lin(\vx)]
    \]
    for all $k \in \{1, \ldots, d-1\}$, and naturally, 
    \[ 
    \Pr_{\vz \sim \U(\vy^{(d)})}[\Lin(\vz) = \Lin(\vx)] = 1.
    \]
    Moreover, $\opt(\Lin, \vx, \delta) = k$ if and only if $\vy^{(k)}$ is a $\delta$-SR for $\vx$ but $\vy^{(k-1)}$ is not. 
 \end{theorem8}

    Before we prove~\Cref{lemma:greedy}, let us prove an auxiliary lemma that will also be useful for proving~\Cref{thm:locally-minimal}.
Given an instance $\vx$ to explain, and a partial instance $\vy \subseteq \vx$, such that $y_i = \bot$, we define the partial instance $\vy \oplus i$ by:
\[ 
    (\vy \oplus i)_j = \begin{cases}
        y_j & \text{if } j \neq i\\
        x_i & \text{otherwise}.
    \end{cases}
\]
\begin{lemma}\label{lemma:positive-score}
Let $\Lin$ be a linear model, $\vx$ an instance and $\vy \subseteq \vx$ a partial instance. Assume a product distribution $\D$. Then, if $i$ is a feature such that $y_i = \bot$, and the feature score $s_i$ holds $s_i \geq 0$, we have 
\[ 
\Pr_{\vz \in \D(\vy \oplus i)}[\Lin(\vz) = \Lin(\vx)] \geq \Pr_{\vz \in \D(\vy)}[\Lin(\vz) = \Lin(\vx)].
\]
\end{lemma}
\begin{proof}[Proof of~\Cref{lemma:positive-score}]
Let $(w_1, \ldots, w_d)$ be the weights of $\Lin$ and $t$ its threshold. Then, let us assume without loss of generality that $\Lin(\vx) = 1$, as the case $\Lin(\vx) = 0$ is analogous. 
We thus have that 
\[ 
    s_i = w_i \cdot (2x_i - 1) \geq 0,
\]
from where $s_i = w_i$ if $x_i = 1$ and $s_i = -w_i$ if $x_i = 0$.
Let us denote by $S$ the set of features $j$ such that $y_j = x_j \neq \bot$, and define 
\[ 
    t' := t - \sum_{j \in S} y_j w_j.
\]
We can then rewrite the probability of interest as
\[ 
    \Pr_{\vz \sim \D(\vy)}[\Lin(\vz) = \Lin(\vx)] =   \Pr_{\vz \sim \D(\vy)}[\Lin(\vz) = 1] = \Pr_{\vz \sim \D(\vy)}\left[\sum_{j \not\in S} z_j w_j \geq t'\right].
\]
Let us define the following two amounts:
\[ 
    \mathcal{A} := \Pr_{\vz \sim \D(\vy)}\left[\sum_{j \not\in S, j \neq i} z_j w_j \geq t' - w_i\right],
\]
\[ 
    \mathcal{B} := \Pr_{\vz \sim \D(\vy)}\left[\sum_{j \not\in S, j \neq i} z_j w_j \geq t'\right].
\]
If $p_i$ is the probability of feature $i$ under $\D$, then we have 
\[ 
    \Pr_{\vz \sim \D(\vy)}[\Lin(\vz) = \Lin(\vx)]  =  \Pr_{\vz \sim \D(\vy)}\left[\sum_{j \not\in S} z_j w_j \geq t'\right] = p_i \cdot \mathcal{A} + (1-p_i)\cdot \mathcal{B}.
\]
Now we proceed by cases on $x_i$. If $x_i = 1$, then $s_i = w_i$, and thus we know $w_i \geq 0$, from where $t' - w_i \leq t'$ and thus $\mathcal{A} \geq \mathcal{B}$. Moreover, as $x_i = 1$, we conclude 
\begin{align*}
    \Pr_{\vz \sim \D(\vy \oplus i)}[\Lin(\vz) = \Lin(\vx)] &= \mathcal{A}\\
     &= p_i \cdot \mathcal{A} + (1-p_i)\cdot \mathcal{A}\\
    &\geq p_i \cdot \mathcal{A} + (1-p_i)\cdot \mathcal{B}\\
    &= \Pr_{\vz \sim \D(\vy)}[\Lin(\vz) = \Lin(\vx)].
\end{align*}

Similarly, if $x_i = 0$, then $s_i = -w_i$, and thus we know $w_i \leq 0$, from where $t' - w_i \geq t'$ and thus $\mathcal{A} \leq \mathcal{B}$. As $x_i = 0$, we conclude 
\begin{align*}
    \Pr_{\vz \sim \D(\vy \oplus i)}[\Lin(\vz) = \Lin(\vx)] &= \mathcal{B}\\
     &= p_i \cdot \mathcal{B} + (1-p_i)\cdot \mathcal{B}\\
    &\geq p_i \cdot \mathcal{A} + (1-p_i)\cdot \mathcal{B}\\
    &= \Pr_{\vz \sim \D(\vy)}[\Lin(\vz) = \Lin(\vx)].
\end{align*}
This concludes the proof.
\end{proof}

Similarly, for any partial instance $\vy$ such that $y_i \neq \bot$, we can define the partial instance $\vy \ominus i$ as 
\[ 
    (\vy \ominus i)_j = \begin{cases}
        y_j & \text{if } j \neq i\\
        \bot & \text{otherwise}.
    \end{cases}
\]
The proof of~\Cref{lemma:positive-score}, but reversing signs, yields the following lemma. 
\begin{lemma}\label{lemma:negative-score}
    Let $\Lin$ be a linear model, $\vx$ an instance and $\vy \subseteq \vx$ a partial instance. Assume a product distribution $\D$. Then, if $i$ is a feature such that $y_i \neq \bot$, and the feature score $s_i$ holds $s_i \leq 0$, we have 
\[ 
\Pr_{\vz \in \D(\vy \ominus i)}[\Lin(\vz) = \Lin(\vx)] \geq \Pr_{\vz \in \D(\vy)}[\Lin(\vz) = \Lin(\vx)].
\]

\end{lemma}

Now, in order to prove~\Cref{lemma:greedy}, which makes two claims, we will split it into two separate lemmas. 
\begin{lemma}[Part 1 of~\Cref{lemma:greedy}]\label{lemma:greedy-1}
    Given a linear model $\Lin$, and an instance $\vx$, if $\vy^{(0)}, \ldots, \vy^{(d)}$ are the partial instances of $\vx$ such that $\vy^{(k)} \subseteq \vx$ is defined only in the top $k$ features of maximum score, then
    \[ 
        \Pr_{\vz \sim \U(\vy^{(k+1)})}[\Lin(\vz) = \Lin(\vx)] \geq \Pr_{\vz \sim \U(\vy^{(k)})}[\Lin(\vz) = \Lin(\vx)]
    \]
    for all $k \in \{0, \ldots, d-1\}$, and naturally, 
    \[ 
    \Pr_{\vz \sim \U(\vy^{(d)})}[\Lin(\vz) = \Lin(\vx)] = 1.
    \]
\end{lemma}
\begin{proof}[Proof of~\Cref{lemma:greedy-1}]
Let us assume without loss of generality that the features are already sorted decreasingly in terms of score, so \[s_1 \geq s_2 \geq \cdots \geq s_d.\]
This way, we have that $\vy^{(k)} \subseteq \vx$ is defined as follows:
\[ 
    y^{(k)}_i = \begin{cases}
        x_i & \text{if } i \leq k\\
        \bot & \text{otherwise}.
    \end{cases}
\]
The proof now requires considering two cases. First, if $s_{k+1} \geq 0$, then we can apply~\Cref{lemma:positive-score} to conclude that
\[ 
    \Pr_{\vz \sim \U(\vy^{(k+1)})}[\Lin(\vz) = \Lin(\vx)] \geq \Pr_{\vz \sim \U(\vy^{(k)})}[\Lin(\vz) = \Lin(\vx)].
\]
We will now show that if $s_{k+1} < 0$, then 
\[ 
    \Pr_{\vz \sim \U(\vy^{(k+1)})}[\Lin(\vz) = \Lin(\vx)] = 1,
\]
which will be enough to conclude. Indeed, as $s_{k+1} < 0$, we have that 
\[
   0 > s_{k+2} \geq s_{k+3} \geq \cdots \geq s_d, 
\]
from where we can repeatedly apply~\Cref{lemma:negative-score} to deduce 
\[ 
    \Pr_{\vz \sim \U(\vy^{(k+1)})}[\Lin(\vz) = \Lin(\vx)]  \geq \Pr_{\vz \sim \U(\vy^{(k+2)})}[\Lin(\vz) = \Lin(\vx)] \geq \cdots \geq \Pr_{\vz \sim \U(\vy^{(d)})}[\Lin(\vz) = \Lin(\vx)] = 1.
\]
This concludes the proof.
\end{proof}
\begin{lemma}[Part 2 of~\Cref{lemma:greedy}]\label{lemma:greedy-2}
    Given a linear model $\Lin$, and an instance $\vx$, if $\vy^{(0)}, \ldots, \vy^{(d)}$ are the partial instances of $\vx$ such that $\vy^{(k)} \subseteq \vx$ is defined only in the top $k$ features of maximum score, then $\opt(\Lin, \vx, \delta) = k$ if and only if $\vy^{(k)}$ is a $\delta$-SR for $\vx$ but $\vy^{(k-1)}$ is not. 
\end{lemma}

In order to prove~\Cref{lemma:greedy-2}, we will use a separate lemma. Let us define, for every $k \in [d]$ the set $P_k$ as the set of partial instances $\vy \subseteq \vx$ such that $\vy$ has $k$ defined features.
\begin{lemma}\label{lemma:greedy-3}
    For any $k \in [d]$, we have 
    \[
        \Pr_{\vz \sim \U(\vy^{(k)})}[\Lin(\vz) = \Lin(\vx)] = \max_{\vy \in P_k} \Pr_{\vz \sim \U(\vy)}[\Lin(\vz) = \Lin(\vx)].
    \]
\end{lemma}

Let us show immediately how~\Cref{lemma:greedy-2} can be proved using~\Cref{lemma:greedy-3}.
\begin{proof}[Proof of~\Cref{lemma:greedy-2}]
  For the forward direction, assume that $\opt(\Lin, \vx, \delta) = k$. Then, by definition, we have that there exists a $\delta$-SR $\vy^\star$ for $\vx$ such that $\vy^\star$ has $k$ defined features.
    By~\Cref{lemma:greedy-3}, we have that 
    \[ 
        \Pr_{\vz \sim \U(\vy^{(k)})}\left[\Lin(\vz) = \Lin(\vx)\right] \geq \Pr_{\vz \sim \U(\vy^\star)}\left[\Lin(\vz) = \Lin(\vx)\right] \geq \delta,
    \]
    and thus $\vy^{(k)}$ is a $\delta$-SR for $\vx$. On the other hand, if $\vy^{(k-1)}$ were to be a $\delta$-SR for $\vx$, then we would have $\opt(\Lin, \vx, \delta) \leq k-1$, a contradiction.
    For the backward direction, assume that $\vy^{(k)}$ is a $\delta$-SR for $\vx$ but $\vy^{(k-1)}$ is not. Then, by~\Cref{lemma:greedy-3}, we have that 
    \[
        \delta > \Pr_{\vz \sim \U(\vy^{(k-1)})}[\Lin(\vz) = \Lin(\vx)] = \max_{\vy \in P_{k-1}} \Pr_{\vz \sim \U(\vy)}[\Lin(\vz) = \Lin(\vx)],
    \]
    from where $\opt(\Lin, \vx, \delta) > k-1$, and because $\vy^{(k)}$ is a $\delta$-SR for $\vx$, we have $\opt(\Lin, \vx, \delta) \leq k$; we conclude that $\opt(\Lin, \vx, \delta) = k$.
\end{proof}

It thus only remains to prove~\Cref{lemma:greedy-3}.
\begin{proof}[Proof of~\Cref{lemma:greedy-3}]
Let $w_1, \ldots, w_d$ be the weights of $\Lin$, and $t$ its threshold. Let us use the $\oplus, \ominus$ notation defined in~\Cref{lemma:positive-score,lemma:negative-score}. 
   We will prove something slightly stronger than~\Cref{lemma:greedy-3}: that if $i$ and $j$ are features such that $s_i \leq s_j$, then for any partial instance $\vy$ such that $y_i \neq \bot$ and $y_j = \bot$, we have
    \[
        \Pr_{\vz \sim \U(\vy \ominus i \oplus j)}[\Lin(\vz) = \Lin(\vx)] \geq \Pr_{\vz \sim \U(\vy)}[\Lin(\vz) = \Lin(\vx)].
    \]
    If we prove this, then we can apply it repeatedly to deduce~\Cref{lemma:greedy-3}. To prove the claim, we start by defining
    \[ 
        S = \{ \ell \mid y_\ell \neq \bot \} \setminus \{i \},
    \]
    and 
    \[ 
        t' = t - \sum_{\ell \in S} y_\ell w_\ell.
    \]
   We will also assume without loss of generality that $\Lin(\vx) = 1$ since the other case is analogous. We can then rewrite the probabilities of interest as follows, using notation $\bar{S} := [d] \setminus S$:
   \[
    \Pr_{\vz \sim \U(\vy \ominus i \oplus j)}[\Lin(\vz) = \Lin(\vx)] = \Pr_{\vz \sim \U(\vy \ominus i \oplus j)}\left[\sum_{\ell \in \bar{S} \setminus \{i, j\}} z_\ell w_\ell  + x_j w_j + z_i w_i\geq t'\right],
   \]
   \[
    \Pr_{\vz \sim \U(\vy)}[\Lin(\vz) = \Lin(\vx)] = \Pr_{\vz \sim \U(\vy)}\left[\sum_{\ell \in \bar{S} \setminus \{i, j\}} z_\ell w_\ell  + x_i w_i + z_j w_j\geq t'\right].
   \]
    Let us write $t^\star = t' - \sum_{\ell \in \bar{S} \setminus \{i, j\}} z_\ell w_\ell$, and note that $t^\star$ is a random variable. With this notation, it remains to prove that 

    \begin{align*}
        &\Pr_{z_i, t^\star}[z_i w_i + x_j w_j \geq t^\star] \geq \Pr_{z_i, t^\star}[z_j w_j + x_i w_i \geq t^\star] \\
        \iff& \frac{1}{2}\Pr_{t^\star}[w_i + x_j w_j \geq t^\star] + \frac{1}{2}\Pr_{t^\star}[x_j w_j \geq t^\star] \geq \frac{1}{2}\Pr_{t^\star}[w_j + x_i w_i \geq t^\star] + \frac{1}{2}\Pr_{t^\star}[x_i w_i \geq t^\star] \\
        \iff& \Pr_{t^\star}[w_i + x_j w_j \geq t^\star] + \Pr_{t^\star}[x_j w_j \geq t^\star] \geq \Pr_{t^\star}[w_j + x_i w_i \geq t^\star] + \Pr_{t^\star}[x_i w_i \geq t^\star].
    \end{align*}
    We will prove the last inequality by cases, recalling that $s_j \geq s_i$ and thus $w_j (2x_j - 1) \geq w_i (2x_i - 1)$.
    \begin{itemize}
        \item (\textbf{Case 1:} $x_i = 1, x_j = 1$) The desired inequality is
        \begin{align*}
            \Pr_{t^\star}[w_i + w_j \geq t^\star] + \Pr_{t^\star}[w_j \geq t^\star] &\geq \Pr_{t^\star}[w_j + w_i \geq t^\star] + \Pr_{t^\star}[w_i \geq t^\star]\\
            \iff \Pr_{t^\star}[w_j \geq t^\star] &\geq \Pr_{t^\star}[w_i \geq t^\star],
        \end{align*}
        which is true since $s_j \geq s_i$ implies $w_j \geq w_i$ given $x_i = x_j = 1$.
        \item (\textbf{Case 2:} $x_i = 1, x_j = 0$) The desired inequality is
        \begin{align*}
            \Pr_{t^\star}[w_i \geq t^\star] + \Pr_{t^\star}[0\geq t^\star] &\geq \Pr_{t^\star}[w_j + w_i \geq t^\star] + \Pr_{t^\star}[w_i \geq t^\star]\\
            \iff \Pr_{t^\star}[0 \geq t^\star] &\geq \Pr_{t^\star}[w_j + w_i \geq t^\star],
        \end{align*}
        which is true since $s_j \geq s_i$ implies $-w_j \geq w_i$ given $x_i = 1, x_j = 0$, and thus $w_i + w_j \leq 0$.
        \item (\textbf{Case 3:} $x_i = 0, x_j = 1$) The desired inequality is 
        \begin{align*}
            \Pr_{t^\star}[w_i + w_j \geq t^\star] + \Pr_{t^\star}[w_j \geq t^\star] &\geq \Pr_{t^\star}[w_j \geq t^\star] + \Pr_{t^\star}[0 \geq t^\star],\\
            \iff \Pr_{t^\star}[w_i + w_j \geq t^\star] &\geq \Pr_{t^\star}[0\geq t^\star],
        \end{align*}
        which is true since $s_j \geq s_i$ implies $w_j \geq -w_i$ given $x_i = 0, x_j = 1$, and thus $w_i + w_j \geq 0$.
        \item (\textbf{Case 4:} $x_i = 0, x_j = 0$) The desired inequality is
        \begin{align*}
            \Pr_{t^\star}[w_i \geq t^\star] + \Pr_{t^\star}[0 \geq t^\star] &\geq \Pr_{t^\star}[w_j \geq t^\star] + \Pr_{t^\star}[0 \geq t^\star],\\
            \iff \Pr_{t^\star}[w_i \geq t^\star] &\geq \Pr_{t^\star}[w_j \geq t^\star],
        \end{align*}
        which is true since $s_j \geq s_i$ implies $-w_j \geq -w_i$ given $x_i = x_j = 0$, and thus $w_i \geq w_j$.
    \end{itemize}
\end{proof}
\Cref{lemma:greedy} now follows directly from~\Cref{lemma:greedy-1} and~\Cref{lemma:greedy-2}, and the sketch proof of~\Cref{thm:locally-minimal} can be completed as we now have proved~\Cref{lemma:positive-score} and~\Cref{lemma:negative-score}.